\documentclass[preprint,12pt]{elsarticle}

\usepackage{amssymb}
\usepackage{amsthm}
\usepackage{amsmath}
\usepackage{tikz-cd}

\newtheorem{thm}{Theorem}
\newtheorem{prop}{Proposition}
\newtheorem{cor}{Corollary}

\theoremstyle{definition}
\newtheorem{defn}{Definition}

\theoremstyle{definition}
\newtheorem{rem}{Remark}

\theoremstyle{definition}
\newtheorem{exam}{Example}

\journal{Journal of Logical and Algebraic Methods in Programming}

\begin{document}

\begin{frontmatter}



\title{The generalised distribution semantics and projective families of distributions\tnoteref{LMU}}


\author{Felix Weitkämper}
\ead{felix.weitkaemper@lmu.de}

 \affiliation{organization={Institut {für} Informatik, Ludwig-Maximilians-Universität München},
addressline={Oettingenstr.~67}, 
             city={80538 München},
             country={Germany}}

\tnotetext[LMU]{This contribution was supported by LMUexcellent, funded
by the Federal Ministry of Education and Research (BMBF) and the Free State
of Bavaria under the Excellence Strategy of the Federal Government and the
Länder.}

\begin{abstract}
  We generalise the distribution semantics underpinning probabilistic logic programming by distilling its essential concept, the separation of a free random component and a deterministic part.
  This abstracts the core ideas beyond logic programming as such to encompass frameworks from probabilistic databases, probabilistic finite model theory and discrete lifted Bayesian networks.
  To demonstrate the usefulness of such a general approach, we completely characterise the projective families of distributions representable in the generalised distribution semantics and we demonstrate both that large classes of interesting projective families cannot be represented in a generalised distribution semantics and that already a very limited fragment of logic programming (acyclic determinate logic programs) in the deterministic part suffices to represent all those projective families that are representable in the generalised distribution semantics at all. 
\end{abstract}


\begin{highlights}
\item The separation of free random and deterministic parts is a \\key modelling principle
\item There are inherent limitations to the expressivity of such models
\item Probabilistic answer set or disjunctive programs provably \\add expressivity
\end{highlights}


\end{frontmatter}


\section{Introduction}
The distribution semantics, first introduced explicitly by Poole  \cite{Poole93} and Sato \cite{Sato95}, kickstarted the development of probabilistic logic programming, a paradigm that extends traditional logic programming with probabilistic primitives to enable probabilistic relational programming with recursion. By cleanly separating the probabilistic part from the deterministic part, the distribution semantics allows the use of techniques developed over decades of logic programming research, such as negation as failure which unlocks recursion as a programming tool.

From this point of view, probabilistic logic programming is a specific set-up of logic programs over independent probabilistic facts. There is no intrinsic reason, though, why the fundamental principle of separating probabilistic and logical components of a statistical relational formalism should be limited to this specific set-up. In this paper, the distribution semantics is studied as the abstract concept of defining a statistical relational specification by specifying an independent probabilistic part and an arbitrary deterministic part on top of that.

There are several motivations for studying the concept in this generality.

The original motivation for the distribution semantics comes from probabilistic logic programming. In its classical formulation, this is considered to be a Datalog program over independent probabilistic facts.
However, the logic programming paradigm includes far more than Datalog; its main proponent, Prolog, is a Turing-complete programming language whose support for metaprogramming via higher-order predicates is a key feature.
Already today the main implementations of probabilistic logic programming such as cplint or ProbLog  support higher-order predicates and meta-calls \citep{deRaedtK15}.

The concept of the distribution semantics is also key to probabilistic databases in the guise of the tuple-independent database model \citep{SuciuORK11}. As Datalog is but one of several query languages considered in the database community, combining the tuple-independent database model with different query logics is very natural. In particular, aggregates are commonly supported by real-life database query languages without being expressible in the classical probabilistic logic programming concept.

On the theoretical side, Cozman and Maua's  \citep{CozmanM19} probabilistic finite model theory centres around evaluating the expressivity of different logics above independent probabilistic facts.
By evaluating the distribution semantics as a general concept, one can distinguish sharply between limitations on expressivity induced by the logic employed in the deterministic part as opposed to intrinsic limitations occasioned by the separation of logic and probability.

Finally, sometimes models that appear to consist of several probabilistic layers can in fact be reduced to a single independent probability distribution. This is true for lifted Bayesian networks based on discrete conditional probability tables, for instance.
One such case where the logic employed is very different from classical first-order or fixed-point logic is Koponen's \cite{Koponen20} lifted Bayesian networks, which are formulated in terms of conditional probability logic.
This logical language allows the expression of statistical statements within the defining formulas of intensional predicates.

To show the potential of the general framework, we characterise the \emph{projective families of distributions} obtainable in the generalised distribution semantics.

Projectivity was recently introduced to the artificial intelligence literature \citep{JaegerS18,JaegerS20} as a strong condition guaranteeing that marginal probabilities are independent of the domain into which the constants invoked in a query are embedded.

However, they have been studied for decades in the field of \emph{pure inductive logic}, where they are used to characterise degrees of belief that rational agents could adopt about the world they might be inhabiting \citep{Carnap50,ParisV15}.
While their focus has traditionally been on unary signatures, general polyadic signatures have recently been investigated in more detail \citep{RonelV16}. 

Harnessing techniques developed in probability theory, Jaeger and Schulte \citep{JaegerS20} showed that projective families of distributions can be represented by exchangeable arrays.  
However, existing statistical relational formalisms have proven unable to express the wide range of possible projective families \citep{Weitkaemper21,MalhotraS22}, and that in particular probabilistic logic programs are restricted to a narrow subclass of projective families of distributions. 
Such results have heretofore been obtained for classical probabilistic logic programming from the asymptotic theory of the concrete fixed-point logic involved in the deterministic part \citep{Weitkaemper21}.
The zero/one laws of finite model theory, on which such arguments are based, are very brittle, though.
A dependency on, say, the number of domain elements being even, would immediately invalidate such arguments, although they are easy to represent using database aggregates or Prolog-style metaprogramming. 

The generalised distribution semantics introduced here allows us to show that the reasons for this do not just lie in limitations of the concrete logical and probabilistic framework used, but are inherent to the underlying concept of neatly dividing logic and probability into loosely coupled components.    
We build on recent work from pure inductive logic \citep{RonelV16} to obtain our characterisation, and we see that the abstract and general assumptions of the generalised distribution semantics already severely limit the representable projective families of distributions.

\section{Frameworks}\label{Fw}

In this section, we introduce our main framework, the generalised distribution semantics, and we provide an alternative formulation of projectivity that integrates well with the concept.

\subsection{Preliminaries}

Our setting is that of finite relational \emph{signatures}, which are finite sets of \emph{relation symbols} of given natural number \emph{arities}.
For such a signature $L$ and finite set $D$, an \emph{$L$-structure} with \emph{domain} $D$ is given by a map that allocates to every $n$-ary relational symbol $R$ in $L$ a subset of $D^n$.

Expressions of the form $R(a_1, \dots, a_n)$, for relation symbols $R$ of arity $n$ and $a_1, \dots, a_n \in D$ are known as \emph{ground atoms}; they are true in a structure $\omega$ with domain $D$, written as $\omega \models R(a_1, \dots, a_n)$, if $(a_1,\dots,a_n)$ lies in the image of $R$ under the map defining $\omega$. In this case, $\omega$ is said to be a \emph{model} of $R(a_1, \dots, a_n)$.
Ground atoms and their negations are known as \emph{ground literals}, and their truth values as well as those of more general Boolean combinations (\emph{ground formulas}) are given as usual in first-order logic.
We use notions such $\varphi(\vec{a})$ for a ground formula $\varphi$ to express that only domain elements in  $\vec{a}$ occur in $\varphi$.

If $\omega'$ and $\omega$ are $L$-structures with domains $A$ and $B$ respectively, then an \emph{embedding} from $\omega'$ to $\omega$ is an injective map $\iota$ from $A$ to $B$ such that for all $a_1, \dots, a_n \in A$ and all $n$-ary relation symbols $R$ in $L$, $R(a_1, \dots, a_n)$ is true in $\omega'$ if and only if $R(\iota(a_1), \dots, \iota(a_n))$ is true in $\omega'$. 
  
\begin{defn}
  When introducing the following notation for $L$-structures, $A \subseteq B$ are sets and $L' \subseteq L$ are signatures.
  \begin{itemize}
    \item If $\omega$ is an $L$-structure with domain $B$, then the \emph{restriction} of $\omega$ to $A$ is the $L$-structure $\omega_A$ on $A$ for which $\omega_A \models R(\vec{a})$ if and only if $\omega \models R(\vec{a})$, for any relation symbol $R$ of $L$ and any tuple $\vec{a}$ of elements of $A$. In this situation, $\omega$ is called an \emph{extension} of $\omega_A$ to $B$. The inclusion map is always an embedding from   $\omega_A$ to $\omega$. 
    \item If $\omega$ is an $L$-structure with domain $A$, then the \emph{reduct} of $\omega$ to $L'$ is the $L'$-structure $\omega^{L'}$ with domain $A$ for which $\omega^{L'}\models R(\vec{a})$ if and only if $\omega \models R(\vec{a})$, for any relation symbol $R$ of $L'$ and any tuple $\vec{a}$ of elements of $A$. In this situation, $\omega$ is called an \emph{expansion} of $\omega^{L'}$ to $L$.      
  \end{itemize}
\end{defn}

\begin{defn}
  For any signature $L$ and finite domain $D$ we define ${\Omega_D}^L$ to be the set of $L$-structures with domain $D$.
  
  As an element of ${\Omega_D}^L$, an $L$-structure with domain $D$ is referred to as an \emph{$L$-world with domain $D$}.
  A \emph{random $L$-world} with domain $D$ is a probability distribution over the set of $L$-structures with domain $D$.

For any injective function of finite sets $\iota:A\rightarrow B$, $\Omega^L(\iota):{\Omega_B}^L\rightarrow {\Omega_A}^L$ maps every $L$-structure $\omega$ with domain  $B$ to the $L$-structure $\omega_{\iota}$ on domain $A$, which models $R(\vec{a})$  if and only if  $\omega \models R(\vec{\iota(a)})$.
If $\iota$ is the inclusion map of an $A\subseteq B$, then $\omega_{\iota}$ is just the restriction of $\omega$ to $A$.
\end{defn}

By viewing ${\Omega_D}^L$ as a probability space, a ground formula can be identified with the set of worlds satisfying it. In this way, a random world defines not just the probabilities of individual worlds, but also of ground formulas.

We study not just individual random worlds, but general models defining a random world on any domain. 

\begin{defn}
  An \emph{$L$ family of distributions} $P$ is a map taking any finite set $D$ as input and returning a random $L$-world $P_D$ on $D$. 
\end{defn}


\subsection{Projectivity}
Now we introduce projective families of distributions in a slightly more general way than Jaeger and Schulte \cite{JaegerS18,JaegerS20}. 

\begin{defn}
Let $P$ be an $L$ family of distributions.
Then $P$ is \emph{projective}  if for any two finite sets $D'$ and $D$, any injective map $\iota: D' \hookrightarrow D$ and any $L$-structure $\mathfrak{X}$ on $D'$ the following holds:
\[ P_{D'}(\mathfrak{X}) = P_{D}(\{\omega\in {\Omega_D}^L \mid \omega_{\iota} = \mathfrak{X}\}) \]
\end{defn}

This notion of projectivity is a direct generalisation of the one advanced by Jaeger and Schulte \citep{JaegerS18,JaegerS20}. More precisely, every projective family of distributions in their sense extends uniquely to a projective family of distributions in our sense, and each of our projective families of distributions extend a projective family in their sense \citep[Proposition 1]{Weitkaemper23}.

We illustrate this definition with  toy ProbLog programs expressing homophily:  
Smokers are more likely to be friends with other smokers.
\begin{exam}
Consider this classic solution using a recursive dependency: 
    \begin{verbatim}
        0.3 :: influences(X,Y).
        0.2 :: starts_smoking(X).
        0.2 :: friends(X,Y).
        smokes(X) :- starts_smoking(X).
        smokes(X) :- friends(X,Y), smokes(Y), influences(Y,X).       
    \end{verbatim}
    Consider a domain in which there is only a single individual. 
    In that case, the probability for this individual to start smoking is 0.2, as the second clause can only be lead to smoking if there is a smoking friend.

    In a domain with two individuals, there is a second way for an individual to smoke, namely that the other individual starts smoking, is a friend and then influences that individual. Therefore, the probability of smoking increases. 
    Thus, the model is not projective since the probability of the single-person world in which that person smokes increases by embedding it into a larger domain. 

Now consider another way of modelling the correlation between smoking and friendship:
    \begin{verbatim}
        0.3 :: smokes(X).
        0.1 :: become_friends(X,Y).
        0.3 :: smoke_together(X,Y).
        friends(X,Y) :- become_friends(X).
        friends(X,Y) :- smokes(X), smokes(Y), smoke_together(X,Y).       
    \end{verbatim}
    In this program, the probability for any individual to smoke is always 0.3, regardless of the size of the domain. 
    Similarly, the probability of friendship depends only on the probability that each of two individuals smoke, that they smoke together and that they become friends.
    Each of these are independent of the remainder of the domain, and therefore this model induces a projective family of distributions. 
    More generally, any determinate ProbLog program (whose clause bodies only contain variables also occurring in the head of that clause) induces a projective family of distributions \citep{JaegerS18}.
\end{exam}


Projectivity has broad implications for both learning and reasoning across domain sizes. 
Its definition guarantees immediately that marginal probabilities do not change when computed in differently sized domains, and Jaeger and Schulte \citep{JaegerS18} have shown that under some additional assumptions, projective families of distributions allow for statistically consistent parameter estimation from subdomains.
For several reasons, this is particularly important for statistical relational formalisms such as probabilistic logic programming. 
Firstly, the general specification of a model independently of a fixed domain is one of the main attractions of using a statistical relational model. 
If the behaviour of that model on domains of different sizes is intransparent or undesired, this undermines the generality of the specification.
Secondly, parameter fitting in statistical relational models is generally difficult, as it usually relies on inference \citep[Chapter 7]{deRaedtKNP16}. 
Under the usual complexity-theoretic assumptions, marginal inference itself is generally intractable even for restricted languages \citep{CozmanM18,VandenBroeckKNP21}, and therefore parameter fitting directly on large domains can be infeasible. 
Jaeger and Schulte's results on statistical consistency \cite{JaegerS18} open up the possibility of learning on random sampled subsets of the target domain without distorting the estimated parameters, avoiding inference directly on the large target domain.   

In pure inductive logic, projective families of distributions are studied in the guise of \emph{exchangeable probability functions}, which operate on countably infinite domains \citep{ParisV15}. 
The equivalence of projective families and exchangeable probability functions is well-known, and follows from the equivalence of exchangeable distributions on infinite domains and projective families of distributions. 
In the following, we uniformly adopt the terminology of projective families of distributions throughout, even when referring to concepts from pure inductive logic.  
A thorough technical investigation of these and more general notions of projective families can be found in \cite{Weitkaemper23}.

\subsection{The distribution semantics}

The key idea of the distribution semantics is to split the complex distribution into two parts, one purely probabilistic (`free') and one purely deterministic. We first introduce the probabilistic part.

\begin{defn}
  A \emph{free} $L$-family of distributions is a projective family of distributions $P$ defined from a \emph{weight function} $w:L\rightarrow (0,1)$ by setting \[P_A(\omega) = \left(\prod_{\substack{\vec{a}\in A,R\in L\\ \omega \models R(\vec{a})}} w(R)\right) \times  \left(\prod_{\substack{\vec{a}\in A,R\in L\\ \omega \models \neg R(\vec{a})}} (1 - w(R))\right).\]
\end{defn}

It is easy to see that any such weight function indeed defines  an $L$-family of distributions $P_w$. 

\begin{exam}
  On any node set $D$, consider a random directed graph, with an edge relation $E$ in which there is an probability $p$ that $E(a,b)$ holds for any $a,b \in D$, independently for all pairs $(a,b)$ of nodes in $D$.
  This is a free $L$-family of distributions, where $L = \{E\}$ and $w(E) = p$.
\end{exam}

We now turn to the deterministic part.

\begin{defn}
A \emph{choice of expansions} (from $L'$ to $L$) is a family of maps $\Pi:{\Omega_D}^{L'}\rightarrow {\Omega_D}^L$ for all finite sets $D$ such that $\Pi(\omega)$ expands $\omega$ for all $\omega\in {\Omega_D}^{L'}$.    
\end{defn}

\begin{defn}
  A \emph{generalised probabilistic logic program (or generalised PLP)} $(P,\Pi)$ is an $L$ family of distributions whose data is given by a free $L'$ family of distributions $P$ and a choice of expansions $\Pi$ from $L'$ to $L$, for an $L'\subseteq L$.
  For any finite set $D$ and every $\Delta\subseteq {\Omega_D}^L$, the probability of $\Delta$ under  $(P,\Pi)$ is given by $P_D(\Pi^{-1}(\Delta))$. 
\end{defn}
We can now see the different application areas mentioned in the introduction as special cases of generalised probabilistic logic programs. 
\begin{exam}
  Probabilistic logic programs  under the distribution semantics \citep{Poole93,Sato95,Riguzzi23} are the paradigmatic examples. They are generalised probabilistic logic programs in which the choice of expansions is given by a Datalog program,
  or a program in the fragment of Prolog supported by the probabilistic logic programming language or system used. 

  Queries over tuple-independent probabilistic databases \citep{SuciuORK11} can be seen as generalised probabilistic logic programs in which the choice of expansions is given by an expression in the associated query language.
  
  The relational Bayesian network specifications studied in Cozman and Maua's probabilistic finite model theory \citep{CozmanM18,CozmanM19}  are generalised probabilistic logic programs in which the choice of extensions is given by a first-order formula. Analogous to the work done in finite model theory, it would be very natural in this context to study generalised probabilistic logic programs whose choices of extensions are given by other logical formalisms such as counting logics, higher-order or fixed-point logics.

  Finally, Koponen's lifted Bayesian networks \cite{Koponen20} are generalised probabilistic logic programs whose choice of extensions are given by formulas in conditional probability logic \citep[Proposition 1]{Weitkaemper22a}.
\end{exam}

A particularly simple subclass are those whose choice of expansions is given by quantifier-free formulas.
\begin{defn}
  A choice of expansions $\Pi$ from $L'$ to $L$ is \emph{determinate} if for every $R \in L\setminus L'$ there is a quantifier-free $L'$-formula $\varphi_R$ such that $R(a_1, \dots, a_n)$ is true in  $\Pi(\omega)$ if and only if $\varphi_R(a_1, \dots, a_n)$ is true in $\omega$. 

  A \emph{determinate} probabilistic logic program is a generalised probabilistic logic program whose  choice of expansions is determinate. 
\end{defn}

Weitk\"amper \cite{Weitkaemper21} showed that determinate probabilistic logic programs correspond exactly to those whose choice of expansions is given by a determinate logic program, which can even be chosen to be acyclic, and that all such generalised probabilistic logic programs are projective.

\section{Classification of projective generalised probabilistic logic programs}
\subsection{Strong independence property}\label{subsec:SIP}

We identify projective generalised PLP as those satisfying the \emph{strong independence property}, first isolated in the context of pure inductive logic \citep{ParisV15} by Ronel and Vencovsk\'{a} \cite{RonelV16}.

An important auxiliary concept in the analysis is the $g$-\emph{trace}, which is usually defined in terms of formulas satisfied by a given random world:

\begin{defn}
  Let $\omega$ be an $L$-world. Then the $g$-ary (syntactic) \emph{trace}   $\mathrm{tr}_g(\omega)$ of $\omega$ is defined as the set of all ground literals $\varphi$ with at most $g$ constants that hold in $\omega$. A $g$-ary trace over a domain $D$ is a $g$-ary trace over any $L$-world with domain $D$. A $g$-ary trace over $L$ is a trace over the domain $\{1,\dots,g\}$. 

  Let $\varphi$ be a quantifier-free $L$-formula whose variables have been ground to elements of a domain. Then $\varphi$ \emph{mentions} a tuple $a_1,\dots,a_n$ if there is an atomic subformula $R(b_1,\dots,b_m)$ of $\varphi$ such that $\{a_1,\dots,a_n\}\subseteq \{b_1,\dots,b_m\}$. 
\end{defn}

\begin{exam}
  Consider a 2-coloured directed graph $G$, equipped with a loop-free binary edge relation $E$ and a  unary colour relation $C$, where  $C(a)$ denotes one colour and  $\neg C(a)$ the other.
  Then the 1-ary trace of $G$ is the collection of all literals of the form $C(a)$ or $\neg C(a)$ for nodes $a$ in $G$, denoting the colour of every node, as well as the set of literals $\neg E(a,a)$ for nodes $a$, expressing that $G$ is loop-free.
  The 2-ary trace of $G$ includes all literals of the type $E(a,b)$ or $\neg E(a,b)$, as well as those included in the 1-trace. The 2-trace therefore completely specifies the coloured graph $G$.
\end{exam}
Note that the $k$-trace of an $L$-world completely specifies that world, where $k$ is at least the highest arity occurring in $L$. 

Since semantic concepts fit better into our framework than criteria defined in terms of quantifier-free formulas, we give equivalent semantic notions:

\begin{defn}
  Let $\omega$ be an $L$-structure. Then the $g$-ary (semantic) trace of $\omega$ is defined as the set of all worlds $\omega'$ on the same domain as $\omega$ such that for every subset $D$ of that domain of cardinality not exceeding $g$,  $\omega'_D = \omega_D$. A $g$-ary trace over a domain $D$ is a $g$-ary trace over an $L$-world with domain $D$.

  Let $\varphi$ be a set of $L$-structures with domain $D$. Then $\varphi$ \emph{mentions} a tuple of distinct elements $a_1,\dots,a_{n}$ of $D$ if there are $\omega_1$ and $\omega_2$ such that ${\omega_1}_{D'} = {\omega_2}_{D'}$ for all $D' \subseteq D$ with $\{a_1,\dots,a_{n}\}\nsubseteq D'$ and $\omega_1 \in \varphi$, but $\omega_2 \notin \varphi$.    
\end{defn}

\begin{prop}\label{prop_trace}
  The semantic trace of  a possible world  $\omega$ are exactly the models of the syntactic trace of $\omega$.

  Whenever a formula $\varphi$ does not (syntactically) mention a tuple, then the models of $\varphi$ do not (semantically) mention it. When a set does not mention a tuple semantically, this set is the set of models of a sentence which does not (syntactically) mentions that tuple.  
\end{prop}

\begin{proof}
  We first show the statement for traces.
  Let $\theta$ be the syntactic trace of $\omega$.
  Then for any world $\omega'$ satisfying $\theta$ and every subset $D$ of cardinality $g$, ${\omega'}_{D}$ has the same $g$-trace over $D$, namely the restriction of $\theta$ to $D$, which completely specifies ${\omega'}_D$.
  Conversely, if ${\omega'}_D = \omega_D$ for all $D$ of cardinality $g$, then $\omega'$ satisfies the same formulas with entries from $D$ as $\omega$, for all $g$-tuples of entries $D$. This implies that $\omega'$ satisfies $\theta$.

  We now show the statement for mentions.
  Let $\varphi$ not (syntactically)  mention $a_1,\dots, a_n$.
  Then for all atoms $\lambda$ in $\varphi$, there is an $a_\lambda \in \{a_1,\dots, a_n\}$ that does not occur in $\lambda$.
  Let ${\omega_1}_D = {\omega_2}_D$ for all $D$ omitting an $a_i$ and let $\omega_1 \models \varphi$.
  This implies that ${\omega_1}_D$ and ${\omega_2}_D$ agree on the truth value of all atoms whose parameters are contained in such a $D$, in other words, on all those atoms for which there is an $a_\lambda \in \{a_1,\dots, a_n\}$ that does not occur in $\lambda$. Thus ${\omega_1}_D$ and ${\omega_2}_D$ agree on the truth value of all atoms in $\varphi$, and thus on the truth value of $\varphi$ itself. 
  Conversely, let $\tilde{\varphi}$ be a set of worlds that does not mention a tuple $a_1,\dots,a_n$.
  Then let $\varphi$ be defined as follows:\\
   For every $\omega \in \tilde{\varphi}$, let $\varphi_{\omega}$ be the conjunction of the $|D|$-traces of $\omega_D$ for all $D$ which omit at least one $a_i$. 
   Then $\varphi$ is defined as the disjunction of the $\varphi_{\omega}$ for all $\omega\in \tilde{\varphi}$.
  Clearly, $\varphi$ does not mention $a_1,\dots, a_n$. 
  It remains to show that the set of models of $\varphi$ is exactly $\tilde{\varphi}$.
  Every element $\omega$ of $\tilde{\varphi}$ is a model of $\varphi$ since it satisfies $\varphi_{\omega}$. 
  To see that the converse is true, let $\omega'\models \varphi$.
  Then $\omega' \models \varphi_{\omega}$ for an $\omega \in \tilde{\varphi}$.
  This implies that ${\omega'}_D = {\omega}_D$ for all $D$ omitting an $a_i$.
  By the semantic mentioning condition, this implies that $\omega' \in \tilde{\varphi}$.
\end{proof}

\begin{prop}
  A generalised probabilistic logic program defines a projective family of distributions if and only if its associated choice of expansions $\Pi$ commutes with restrictions and extensions, that is, if for any injective $\iota:A\rightarrow B$, the following square commutes:
  \[
\begin{tikzcd}
{\Omega_A^{L'}} \arrow[r, "\Pi "]                    & {\Omega_A^{L}}                    \\
{{\Omega_B}^{L'}} \arrow[u, "\Omega^{L'}(\iota)"] \arrow[r, "\Pi "] & {{\Omega_B}^{L}} \arrow[u, "\Omega^{L}(\iota)"]
\end{tikzcd}
\]

This can also be expressed by saying that in this situation, $\Pi\circ \pi = \pi \circ \Pi$.
\end{prop}
\begin{proof}
  To make the role of generalised probabilistic logic programs in this argument more transparent, we cast this proof in the language of category theory.
  All notions we refer to here can be found in Chapter 1 of Leinster's \cite{Leinster14} textbook, or in any other introduction to category theory.
  In the following, let $\mathrm{SET}_{\mathrm{inj}}$ denote the category of finite sets, with injective maps as morphisms,
  and let MEAS denote the category of finite measure spaces, with measure-preserving maps as morphisms.
  For any category CAT, let $\mathrm{CAT}^{\mathrm{op}}$ denote the opposite category. 
  Then a  projective $L$-family of distributions $P$ is precisely a contravariant functor from  ${\mathrm{SET}_{\mathrm{inj}}}$ to MEAS which extends $\Omega^L$.
  In fact, every projective family of distributions is an  equivalence of categories from  ${\mathrm{SET}_{\mathrm{inj}}}^{\mathrm{op}}$ to its image $\mathrm{Im}(P)$, the subcategory of MEAS whose objects are the measure spaces $P_w(A)$ for a finite set $A$ and whose morphisms are the measure-preserving maps induced by injective functions between finite sets.
  Consider the functor $\Delta$ from $\mathrm{Im}(P)$ to ${\mathrm{SET}_{\mathrm{inj}}}^{\mathrm{op}}$ that maps $P_w(A)$ to $A$ and maps any morphism to the injective function inducing it.
  This is well-defined.
  Indeed, for any $\iota:A \hookrightarrow B$ and any $a \in A$ we can consider the structure $\omega$ with domain $B$ in which for an arbitrary relation $R$, $R(\iota(a),\dots,\iota(a))$ is true and $R$ is false for all other tuples.
  Then $\Omega_L(\iota)(\omega)$ is the structure with domain $A$ in which $R$ holds for $(a,\dots,a)$ and no other tuple.
  Thus $\Omega_L(\iota)$ uniquely identifies $\iota(a)$, for any $a \in A$.
  By construction, $\Delta$ is an inverse of $P_w(A)$, and by projectivity it is indeed a functor.
    
  In particular, the free part $P$ of the logic program induces such an equivalence of categories.
  Hence the generalised PLP is functorial if and only if the map $\Pi^*$ from $\mathrm{Im}(P)$ to MEAS induced by $\Pi$ is functorial
  (where the underlying set of $\Pi^*(({\Omega_D}^{L'},\mu))$ is  ${\Omega_D}^{L'}$ and the probability measure is the pushforward measure of $\mu$ under $\Pi$, i.\ e. $\Pi^*(\mu)(\Delta) = \mu(\Pi^{-1}(\Delta))$).

This is encapsulated in the commutativity of the following diagram,
  \[ \begin{tikzcd}
{(\Omega_A^{L'},\mu_A)} \arrow[r, "\Pi "]                    & {(\Omega_A^{L},\Pi^*\mu_A)}                    \\
{({\Omega_B}^{L'},\mu_B)} \arrow[u, "\Omega^{L'}(\iota)"] \arrow[r, "\Pi "] & {({\Omega_B}^{L},\Pi^*\mu_B)} \arrow[u, "\Omega^{L}(\iota)"]
  \end{tikzcd} \]

  As the maps here coincide with the maps in the requirements of the proposition, the ``only if'' direction follows immediately.

  ``If'': If $\Pi$ satisfies the requirements in the proposition, then this diagram clearly commutes. It suffices therefore to show that the restriction map from $({\Omega_B}^{L},\Pi^*\mu_B)$ to $(\Omega_A^{L},\Pi^*\mu_A)$ is measure-preserving, that is, that for any $\omega\in \Omega_A^{L}$,  \[
  \mu_A(\Pi^{-1}(\{\omega\})) = \mu_B\Pi^{-1}\{(\pi^{-1}(\omega))\}.\]
  However, since by the requirements of the proposition $\Pi^{-1}\{(\pi^{-1}(\omega))\} = \pi^{-1}\{(\Pi^{-1}(\omega))\}$, this follows from the fact that $\pi$ is measure-preserving with respect to $\mu_B$ and $\mu_A$.

\end{proof}

\begin{cor}\label{cor_gtog}
  Let $(P,\Pi)$ be a projective generalised PLP, where $P$ is a free $L'$ family of distributions and $\Pi$ a choice of expansions from $L'$ to $L$.\@
  Let $\omega_1,\omega_2$ be $L'$-structures and let $g \in \mathbb{N}$ such that the $g$-trace of $\omega_1$ coincides with the $g$-trace of $\omega_2$. Then the $g$-trace of $\Pi(\omega_1)$ coincides with the $g$-trace of $\Pi(\omega_2)$. 
\end{cor}
\begin{proof}
  For any $A=\{a_1,\dots,a_g\}$ contained in the intersection of the domains of $\omega_1$ and $\omega_2$, consider the restrictions $\omega_{1,A}$ and $\omega_{2,A}$.
  Since the $g$-traces of $\omega_1$ and $\omega_2$ coincide, $\Pi(\omega_{1,A}) = \Pi(\omega_{2,A})$.
  By the theorem above, ${\Pi(\omega_1)}_A = \Pi(\omega_{1,A})$ and  ${\Pi(\omega_2)}_A = \Pi(\omega_{2,A})$, and since $A$ was arbitrary with cardinality not exceeding $g$, this shows that  the $g$-trace of $\Pi(\omega_1)$ coincides with the $g$-trace of $\Pi(\omega_2)$. 
\end{proof}

\begin{defn}
  Let $L$ be a signature with maximal arity $r$. A projective $L$-family of distributions $P$ satisfies the \emph{Strong Independence Principle (SIP)} if the following holds:
  
  Let $0 \leq g < r$ and let $\varphi$ and $\psi$ be ground quantifier-free formulas with values in a domain $D$ that mention no joint $g+1$-set of constants.
  Furthermore, let $\theta$ be a $g$-ary trace for the elements occurring in both $\varphi$ and $\psi$. Then
  \[
  P_D(\varphi \cap \psi \mid \theta) = P(\varphi \mid \theta) \cdot P(\psi \mid \theta).
  \]
\end{defn}

From Proposition \ref{prop_trace} we can immediately deduce that the SIP is equivalent to what we call \emph{semantic} SIP, where ``trace'' and ``mentioning'' are replaced by ``semantic trace'' and ``semantic mentioning'' respectively.

Ronel and Vencovsk\'{a} \cite[Theorem 1]{RonelV16} give a concrete characterisation of projective distributions with SIP.\@

The projective distributions with SIP are exactly those obtainable as follows:

Given a signature $L$,  for each domain $D = \{a_1, \dots, a_n\}$, the construction proceeds by induction on $g$, starting at $g=1$ and proceeding to the highest arity of relation symbols in $L$.
For every $g$, we construct a distribution over the $g$-traces over $D$.
When $g$ is the highest arity of relation symbols occurring in $L$, a $g$-trace over $D$ completely specifies a world on this domain and therefore a distribution over $g$-traces over $D$ is the same as a distribution over $L$-worlds with domain $D$.

So let $\gamma_1,\dots,\gamma_l$ be an enumeration of the 1-traces over $L$. Then specify $p_1,\dots,p_l \in [0,1]$ with $p_1+ \cdots + p_l = 1$.
For every $a\in D$, choose the 1-trace of $a$ independently, where $\gamma_i$ is chosen with probability $p_i$.
This results in a distribution over the 1-traces over $D$.

Assume we are given a distribution over the $g$-traces over $D$.

We extend this to a distribution over the $g+1$-traces over $D$ by defining conditional on every distribution over $g$-traces a distribution over $g+1$-traces extending that $g$-trace.
We obtain our overall distribution over $g+1$-traces by first choosing a $g$-trace according to the distribution from the last step and then choosing an extension according to the newly-defined distribution.
So specify for every $g$-trace $\theta$ over $\{1,\dots,g+1\}$ whose extensions to $g+1$-traces over $L$ are $\gamma_{\theta,1},\dots,\gamma_{\theta,k}$, real numbers $p_{\theta,1},\dots,p_{\theta,k} \in [0,1]$ such that (1) $p_{\theta,1}+ \cdots + p_{\theta,k} = 1$ for every $g$-trace $\theta$ and (2) such that $p_{\theta,i} = p_{\theta',j}$ whenever the worlds described by $\gamma_{\theta,i}$ and $\gamma_{\theta',j}$ on $\{1,\dots,g+1\}$ are isomorphic (the latter requirement is necessary to ensure exchangeability of the resulting distribution).
Then for every sequence $\vec{a} := a_{i_1},\dots,a_{i_{g+1}}$ of $D$-elements with strictly ascending indices, let $\theta_{\vec{a}}$ be the $g$-trace over $\{1,\dots,g+1\}$ induced by the $g$-trace over $D$ by identifying $j \in \{1,\dots,g+1\}$ with $a_{i_j}$, and choose among the extensions $\gamma_{\theta_{\vec{a}},h}$ of $\theta_{\vec{a}}$ independently and with probability $p_{\theta_{\vec{a}},h}$.
This results in a distribution over the $g+1$-traces over $D$.

The parameters of the construction are the $(p_i)$ and $(p_{\theta,i})$, and choosing different values for these parameters generates all possible projective families of distributions with SIP.\@

\begin{rem}
  If the signature is binary, the projective families satisfying SIP are exactly the relational block models introduced by Malhotra and Serafini \cite{MalhotraS22}.
  Thus SIP distributions can thus also be seen as a higher-arity relational version of stochastic block models \citep{HollandLL83}.
\end{rem}

\begin{exam}
  We illustrate the procedure with a classical relational block model on a signature $L = \{P,E\}$ of coloured graphs, where $P$ is unary and $E$ is binary:
  
  There are four possible 1-traces over $L$, stating whether $P(1)$ is true or false and whether $E(1,1)$ is true or false.
  Let $\gamma_1$ express that both are false,  $\gamma_2$ express that $P(1)$ is true and $E(1,1)$ is false, $\gamma_3$ express that $P(1)$ is false and $E(1,1)$ is true and  $\gamma_4$ express that both are true.
  Thus one can specify $p_1$, $p_2$, $p_3$ and $p_4$, the probabilities of each of the four possibilities.
  If we want to define a distribution over loop-free  graphs in which $P$ is determined completely randomly, we can set $p_1=p_2=0.5$ and $p_3=p_4=0$.
  Then for every pair of nodes $(a,b)$, there are four possibilities for the edge relation: (1) There can be no edge, (2) there is an edge from $a$ to $b$ but not vice versa, (3)  there is an edge from $b$ to $a$ but not vice versa, and (4) there are edges both from $a$ to $b$ and vice versa.
  Say that we want to define a distribution over undirected graphs, and that there should be a higher likelihood for two edges to be connected if both nodes satisfy $P$.
  Then we might set $p_{\theta,1}$ to be 0.3 if $\theta$ implies that both nodes satisfy $P$, and 0.1 if not, and set  $p_{\theta,4}$ to be 0.7 and 0.9 respectively.
  Since we want to enforce only undirected graphs, we set  all  $p_{\theta,2}$ and  $p_{\theta,3}$ to zero.

  Then the overall distribution over coloured graphs on a given node set is defined by first throwing a fair coin for every node to determine whether the node satisfies $P$ or not, and then to go through all pairs of nodes and throw a biased coin to determine whether the pair of nodes is connected by an edge or not. The bias of that coin depends on whether the two nodes both satisfy $P$ or not.

  The SIP now says that if we condition on all the information about a given subgraph, i.e.\ whether their nodes satisfy $P$ and whether they are connected by an edge, then the events of distinct sets of other points being connected to the nodes in that subgraph in some particular configuration are independent.
  This conditioning is important: Consider nodes $a$, $b$ and $c$. Then the probability of $a$ being connected to $b$ or $c$ is higher if $a$ satisfies $P$.
  Thus, the event of $a$ and $b$ being connected is not unconditionally independent on $a$ and $c$ being connected, as knowing the former makes $P(a)$ and thus also the latter event more likely.
  However, as soon as we condition on whether $a$ satisfies $P$ or not, this dependence disappears and the the two events are now conditionally independent as postulated by the SIP.\@

\end{exam}

\subsection{Projective generalised probabilistic logic programs}
In this subsection, we will prove our main result, characterising the distributions induced by projective generalised probabilistic logic programs. To ease reading, we frequently identify a generalised PLP with its induced family of distributions, and call a generalised PLP projective if it induces a projective family of distributions.  
\begin{thm}\label{ProjPLPSIP}
  Every projective generalised PLP satisfies SIP.\@
\end{thm}

\begin{proof}
  We show that every projective generalised PLP satisfies semantic SIP.\@
  So let $(P,\Pi)$ be a projective generalised PLP.\@
  
  Let $\varphi_1,\varphi_2 \subseteq \Omega_D^L$ not mention any joint $g+1$-tuple and let $\theta$ be a $g$-ary trace over a world $\omega(\theta)$ with domain $D$. 
  We show that $\Pi^{-1}(\theta)$ is a $g$-ary trace over $D$ or the empty set, and that   $\Pi^{-1}(\varphi_1)$ and  $\Pi^{-1}(\varphi_2)$ do not mention any joint $g+1$-tuple. Then we can derive the statement from semantic SIP for free distributions.
  \begin{enumerate}
  \item ``$\Pi^{-1}(\theta)$ is a $g$-ary trace over $D$ or the empty set.''
    
    Let $\Pi^{-1}(\theta)$ be nonempty. Then the following equalities demonstrate that $\Pi^{-1}(\theta)$ is a $g$-ary trace over $D$:
    \begin{align*}
      \Pi^{-1}(\theta)
      &= \{\omega \in \Omega_D^{L'} \mid {\Pi(\omega)}_{D_i} = {\omega(\theta)}_{D_i} \forall_{D_i \subseteq D:|D_i|=g}\}\\
      &= \{\omega \in \Omega_D^{L'} \mid {\Pi(\omega_{D_i})} = {\omega(\theta)}_{D_i} \forall_{D_i \subseteq D:|D_i|=g}\}\\
      &= \{\omega \in \Omega_D^{L'} \mid \omega_{D_i} \in \Pi^{-1}({\omega(\theta)}_{D_i}) \forall_{D_i \subseteq D:|D_i|=g}\}\\
      &=  \{\omega \in \Omega_D^{L'} \mid \omega_{D_i} = {\omega(\theta)}_{D_i}^{L'} \forall_{D_i \subseteq D:|D_i|=g}\}
    \end{align*}
  \item ``$\Pi^{-1}(\varphi_1)$ and  $\Pi^{-1}(\varphi_2)$ do not mention any joint $g+1$-tuple''.
    We show that generally whenever $\varphi$ does not mention a $g+1$-tuple, then $\Pi^{-1}(\varphi)$ does not mention a $g+1$-tuple either.
    So let $a_1,\dots,a_{g+1}$ be a tuple of distinct elements of $D$  and $\omega_1$ and $\omega_2$  $L'$-worlds  with domain $D$ such that ${\omega_1}_{D'} = {\omega_2}_{D'}$ for all $D'\subseteq D$ omitting an $a_i$ and $\omega_1\in \Pi^{-1}(\varphi)$.
    


    It remains to show that $\Pi(\omega_2) \in \varphi$.
    By the assumptions on $\varphi$, it suffices to show that ${\Pi({\omega_1})}_{D'} = {\Pi({\omega_2})}_{D'}$ for all $D'\subseteq D$ omitting an $a_i$.

   This follows from 
    \[
    {\Pi({\omega_1})}_{D'} = \Pi({\omega_1}_{D'}) = \Pi({\omega_2}_{D'}) = {\Pi({\omega_2})}_{D'}.
    \]
    
  \item Every free distribution satisfies semantic SIP.\@
    
  \end{enumerate}
  So let $(P,\Pi)$ be a projective generalised PLP, and let $\varphi_1$ and $\varphi_2$ not mention a joint $g$-ary trace. Further let $\theta$ be a $g$-ary trace.
  We want to show that $\varphi_1$ and $\varphi_2$ are conditionally independent over $\theta$.
  The conditional probabilities of $\varphi_1$, $\varphi_2$ and $\varphi_1 \cap \varphi_2$  over $\theta$ under  $(P,\Pi)$ is given by the probabilities of $\Pi^{-1}(\varphi_1)$, $\Pi^{-1}(\varphi_2)$ and $\Pi^{-1}(\varphi_1 \cap \varphi_2)$ over $\Pi^{-1}(\theta)$ respectively. By the analysis above, the strong independence statement follows directly from the strong independence property for the free distribution $P$. 
\end{proof}

Generalised probabilistic logic programs always have a non-zero likelihood of inducing a completely symmetric model, since all random predicates may be simultaneously true or false.
This is formalised in the following definition. 

\begin{defn}
  Let $\theta_{g+1}$ be a $g+1$-ary trace over a world $\omega$ with domain $\{a_1,\dots,a_{g+1}\}$ and let  $\theta_g$ be the $g$-ary trace of $\omega$.
  Then $\theta_g \subseteq \theta_{g+1}$ is a \emph{symmetric extension} if for every permutation $\rho$ of $a_1,\dots,a_{g+1}$, if $\mathrm{tr}_{g}\left(\Omega^L(\rho)(\omega)\right) = \theta_g$, then $\mathrm{tr}_{g+1}\left(\Omega^L(\rho)(\omega)\right) = \theta_{g+1}$ 
  A projective family of distributions $P$ is called \emph{essentially asymmetric} if there is a $g$-ary trace $\theta(a_1,\dots,a_{g+1})$ such that $P_{\{a_1,\dots,a_{g+1}\}}(\theta_g) > 0$ and
  \[
  P_{\{a_1,\dots,a_{g+1}\}}\left(\{\mathrm{tr}_{g+1}(\omega)\textrm{ symmetric extension of }\theta \mid \omega \models \theta\}\right) = 0.
  \]
\end{defn}

\begin{prop}\label{ProjPLPnotEssAsym}
  Let $(P,\Pi)$ be a projective generalised PLP.\@ Then its induced distribution is not essentially asymmetric.
\end{prop}

\begin{proof}
  Note first that since $P$ is a free distribution, every random $L'$-world has non-zero probability, where $L'$ is the signature of the free random predicates. 
Let  $\theta$ be a  $g$-ary trace with domain $\{a_1,\dots,a_{g+1}\}$ such that ${(P,\Pi)}_{\{a_1,\dots,a_{g+1}\}}(\theta) > 0$. 
By Corollary \ref{cor_gtog}, the $g$-trace only depends on the $g$-trace in $L'$.
  So there is a random $L'$ world $\tilde{\omega}$ on $a_1,\dots,a_{g+1}$  such that $\Pi(\omega') \models \theta$ for all $\omega'$ whose $g$-trace coincides with that of $\tilde{\omega}$.
  Then let $\omega$ be the $L'$ world for which the $g$-trace coincides with $\tilde{\omega}$ and all atomic formulas with $g+1$ different entries are false.
  We claim that $\Pi(\omega)$ is not an asymmetric extension of $\tilde{\omega}$.
  Let $\rho$ be a permutation of $a_1,\dots,a_{g+1}$ such that $\mathrm{tr}_{g}\left(\Omega^L(\rho)(\omega)\right) = \theta$.
  Then in particular the $g$-ary $L'$-trace of $\omega$ is invariant under $\rho$.
  This implies that the $g+1$-trace of $\omega$ is also invariant under $\rho$, since all atomic formulas with $g+1$ different entries are false in $\omega$ and having  $g+1$ different entries is conserved under $\rho$.
  Thus the $g+1$-ary trace of $\Pi(\omega)$ is also invariant under $\rho$ as desired.     
\end{proof}

To formulate the other direction of the argument, we need to pass from generalised probabilistic logic programs to their \emph{reducts}.
\begin{defn}
  Let $P$ be an $L$ family of distributions and let $L' \subset L$ be signatures.
  Then the \emph{reduct $P'$ of $P$ to $L'$} is the $L'$ family of distributions mapping a finite set $D$ to the random $L'$-world ${P'}_D$, defined by
  \[
  {P'}_D(\mathfrak{X}) := {P}_D(\omega \in {\Omega_D}^{L} \mid \omega^{L'} = \mathfrak{X}).
  \]
  In other words, the probability of a world under the reduct of $P$ is the probability that $P$ gives to the set of its expansions.
  
  Finally, the reduct of a generalised probabilistic logic program is the reduct of its induced family of distributions.  
\end{defn}

\begin{thm}\label{SIPnotEAisRed}
Every projective family of distributions that has the strong independence property and is not essentially asymmetric is the reduct of a determinate generalised PLP.\@
\end{thm}

\begin{proof}
  We use the explicit characterisation of families of distributions with the strong independence property from Subsection \ref{subsec:SIP}.
  So let the $\gamma_i$ and $\gamma_{\theta,i}$ be the possible traces, as in the discussion in Subsection \ref{subsec:SIP}.
  The goal of the construction is to define a distribution over the $n$-traces for every subset of size $n$ in accordance with the given parameters $(p_i)$ and $(p_{\theta,i})$, for every $n$ not exceeding the highest arity of predicates in $L$.
  
  We have to solve two problems simultaneously:
  \begin{itemize}
  \item We have to define a distribution over the $\gamma$\textsubscript{s} in accordance with the given parameters $(p_i)$ and $(p_{\theta,i})$. This leads to a distribution for every ordered tuple of size $n$.   
  \item We have to define a local ordering on $a_1, \dots, a_n$. Coupled with the solution of the other problem, this results in a distribution for every (unordered) subset.     
  \end{itemize}

  To facilitate our argument, we write $\Pi$ as a determinate logic program, using $\leftarrow$ to denote the clause constructor. 
  They can be read as first-order formulas using \emph{Clark's Completion} \cite{Clark87}, so that if $Q(\vec{x})$ is the head of clauses $Q(\vec{x}) \leftarrow \mathbf{B_1}(\vec{x}), \dots, Q(\vec{x}) \leftarrow \mathbf{B_n}(\vec{x})$, then $Q(\vec{x})$ is true if any of the $\mathbf{B_i}(\vec{x})$ are true.
  Within a clause body, the comma separator is read as a conjunction operator. 
  
  We begin with the first issue, revisiting the classical approach to representing annotated disjunctions in probabilistic logic programming going back to Vennekens \emph{et al.} \cite{VennekensVB04}.
  Assume we want to define a distribution where for any given $a_1,\dots,a_n$, exactly one of
  \[Q_1(a_1,\dots,a_n), \dots, Q_m(a_1,\dots,a_n)\] is true, the probability of $Q_i(a_1,\dots,a_n)$ is $p_i$  and the choices are independent for different $a_1,\dots,a_n$.
We introduce new free $n$-ary predicates $R_i$, $1 \leq i \leq n-1$, with probabilities $w(R_i) := \frac{p_i}{\prod_{j=1}^{i-1} \left(1-w(R_j) \right)}$, and then define in $\Pi$ the following definitions for $Q_i$:
  \begin{align}\label{eqn:LPAD}
    Q_1(\vec{x}) &\leftarrow R_1(\vec{x}). \\
    Q_2(\vec{x}) &\leftarrow R_2(\vec{x}), \neg R_1(\vec{x}). \\
    &\vdots \\
    Q_{m-1}(\vec{x}) &\leftarrow R_{m-1}(\vec{x}), \neg R_{m-2}(\vec{x}), \dots, \neg R_{1}(\vec{x})\\
    Q_m(\vec{x}) &\leftarrow \neg R_{m-1}(\vec{x}), \dots, \neg R_{1}(\vec{x}). 
  \end{align}

  We proceed by induction on $g$. For $g=1$, we introduce auxiliary unary predicates $Q_i$ and $R_i$ for $\gamma_{1},\dots,\gamma_{m}$ as above. We identify $Q_i$ with $\gamma_i$ using the rules
  \[
  P(x) \leftarrow Q_i(x)  
  \]
  whenever an atom $P(a)$ is contained in $\gamma_i$. 
  
  So assume that we have induced the correct distribution on $g$-traces.
  For every $g$-trace $\theta$ over $a_1, \dots, a_{g+1}$ we could now introduce auxiliary $g+1$-ary predicates $Q_{\theta,i}$ and $R_{\theta,i}$ as above to match the prescribed distribution on $\gamma_{\theta,1},\dots, \gamma_{\theta,n}$. 
  
  However, we need to address the second issue, as we currently have conflicting information from the $R_{\theta,i}$ for every permutation of $a_1,\dots, a_{g+1}$.
  Thus, we need to use the information on the validity of free predicates for $a_1, \dots, a_{g+1}$ to induce an ordering and thereby fix a privileged permutation.
  
  Because the distribution is not essentially asymmetric, we can assume without loss of generality that $\gamma_{\theta',1}$ is a symmetric extension of nonzero conditional probability for any $g$-ary trace $\theta'$.
  Furthermore, we choose $\gamma_{\theta',1}$ and  $\gamma_{\theta'',1}$ to be isomorphic extensions whenever $\theta'$ and $\theta''$ are isomorphic.
  
  Let $p_{\min}$ be the minimum of the probabilities of $\gamma_{\theta',1}$ for all $g$-ary traces $\theta'$.
  We form  a $g+1$-ary annotated disjunction of new $g+1$-ary auxiliary predicates $\mathrm{Ord}_{g+1,i}$ with $k$ disjuncts, each of which have equal probability $\frac{1}{k}$. We choose $k$ such that for a given $a_1,\dots,a_{g+1}$ the probability that there is a disjunct $\mathrm{Ord}_{g+1,j}$ and a nontrivial permutation $\rho$  with $\mathrm{Ord}_{g+1,j}(a_1,\dots,a_{g+1}) \land \mathrm{Ord}_{g+1,j}(\rho{a_1},\dots,\rho{a_{g+1}})$  is less than $p_{\min}$. We call that probability  $p_{\mathrm{sym}}$.
  This is always possible, since the probability of two disjuncts coinciding among the fixed number of possible permutations limits to 0 as the number of disjuncts $k$ increases.
  
  In this way, we can assign the entire case of two coinciding disjuncts to the symmetric extension, where there is no problem at all with being unable to define an ordering.
  We mark this case with a specific predicate $Q'_{\theta,1}$. The residual probability of $\gamma_{\theta,1}$ can then be captured precisely by a second predicate $Q_{\theta,1}$, leading to an exact expression of the original distribution. 

  More precisely, we proceed as follows.
  Let $R_{\theta,0}$ be defined by a rule saying that $R_{\theta,0}(a_1,\dots,a_{g+1})$ holds if and only if two of the $\mathrm{Ord}$-disjuncts coincide for permutations of $a_1,\dots,a_{g+1}$.
  Note that $R_{\theta,0}$ itself is permutation invariant, that is, it holds for one permutation of its arguments if and only if it holds for all permutations of its arguments. 
  Whenever $R_{\theta,0}$ is false, we only consider the annotated disjunction over the  $R_{\theta,i}(a_1,\dots,a_{g+1})$ for that permutation $(a_1,\dots,a_{g+1})$ for which $\mathrm{Ord}_{g+1,j}(a_1,\dots,a_{g+1})$ is true for the highest $j$ among permutations.
  Since all permutations have a different such $j$, the maximum is uniquely determined.
  We define $R_{g+1,\max}(a_1,\dots,a_{g+1})$ to be true if and only if  $(a_1,\dots,a_{g+1})$ is the unique permutation with the maximal $j$ such that $\mathrm{Ord}_{g+1,j}(a_1,\dots,a_{g+1})$ is true.
  Whenever $R_{\theta,0}(a_1,\dots,a_{g+1})$ is false,  $R_{g+1,\max}(a_1,\dots,a_{g+1})$ is true for exactly one permutation of $a_1,\dots,a_{g+1}$, and whenever $R_{\theta,0}(a_1,\dots,a_{g+1})$ is true,  $R_{g+1,\max}(a_1,\dots,a_{g+1})$ is false for all permutations of $a_1,\dots,a_{g+1}$.

  Finally we can proceed with the annotated disjunction for \[{Q'}_{\theta,1},Q_{\theta,1},Q_{\theta,2},\dots, Q_{\theta,n},\] where ${Q'}_{\theta,1}$ holds with probability  $p_{\mathrm{sym}}$, the probability of $Q_{\theta,1}$ is the difference between $p_{\theta,1}$ and  $p_{\mathrm{sym}}$, and the atoms in $\gamma_{\theta,1}$ are set to hold whenever ${Q'}_{\theta,1}$ or $Q_{\theta,1}$ holds.
  Let $R_{\theta,0}$ be the auxiliary predicate corresponding to  ${Q'}_{\theta,1}$, and then form the  auxiliary rules and probabilities as before, but appending  $R_{g+1,\max}(a_1,\dots,a_{g+1})$ to every rule as follows (where the arguments have been omitted for readability and are always assumed to be $x_1,\dots, x_{g+1}$).
  \begin{align*}
    {Q'}_{\theta,1} &\leftarrow R_{\theta,0}. \\
    Q_{\theta,1} &\leftarrow R_{\theta,1}, R_{g+1,\max}. \\
    Q_{\theta,2} &\leftarrow R_{\theta,2}, \neg R_{\theta,1}, R_{g+1,\max}. \\
     &\vdots \\
    Q_{\theta,n} &\leftarrow \neg R_{n-1}, \dots, \neg R_{1}, R_{g+1,\max}. 
  \end{align*}

  Since  $R_{g+1,\max}$ holds for exactly one permutation whenever  $R_{\theta,0}$ is false, the rules for $Q_{\theta,1}$ and below fire for exactly one permutation, so there is no conflict between different permutations. When $R_{\theta,0}$ holds, it holds for all permutations, but there is still no conflict since all  $\gamma_{\theta,1}$ are symmetric and different permutations have isomorphic $\theta$, hence isomorphic $\gamma_{\theta,1}$.

\end{proof}

\begin{exam}
  Consider a signature with a single binary relation $P$. Then the possible 1-traces of a 1-element set $\{a\}$ are the 2 traces where $P(a,a)$ either holds or does not. 
  Each pair of these 1-traces can be extended to a 2-trace which additionally specifies which (neither, one or both) of $P(a,b)$ and $P(b,a)$ hold.

  A distribution on these 1-traces simply specifies a probability with which $P(x,x)$ holds for any element $x$. It is modelled by an annotated disjunction with two mutually exclusive unary predicates $Q_1$ and $Q_2$ as in Equation~\ref{eqn:LPAD} as well as the additional rule 
  \[P(x,x) \leftarrow Q_1(x).\]
  Disregarding the issue of fixing an ordering, for every 1-trace $\theta$ of a 2-element set $\{a,b\}$ the four extensions $\gamma_{\theta,1}, \dots, \gamma_{\theta,4}$ could be accommodated by $Q_{\theta,1}, \dots, Q_{\theta,4}$, which are made to be mutually disjoint with the correct probabilities.
  For the sake of this example, assume that for any such $\theta$, in $\gamma_{\theta,1}$ $P(a,b)$ holds, but $P(b,a)$ does not hold. 
  The additional rules could take a form such as
  \[P(x,y) \leftarrow Q_{\theta,1}(x,y), Q_1(x), Q_1(y)\]
  where $\theta$ specifies that $a$ and $b$  satisfy the trace encoded by $Q_1$.
  However, the issue arising now is that in a single possible world, $Q_{\theta,1}(a,b)$ could hold alongside $Q_{\theta,1}(b,a)$.
  In that case, the $Q_{\theta,i}$ carry conflicting information, since one of them specifies that exactly $P(a,b)$ should hold and the other that exactly $P(b,a)$ should hold.
  The logic program as written above would deduce that both hold, which is in fact compatible with neither of the two types.

  The proof overcomes this issue by defining additional auxiliary predicates which probabilisitically enforce a preferred ordering on $\{a,b\}$. If the ordering $(a,b)$ is selected, then $P(a,b)$ holds, and if the ordering $(b,a)$ is selected, then $P(b,a)$ holds. Due to the probability of drawing a completely symmetric set of random facts, it cannot be ruled out that no ordering can be selected; however, the probability of this can be made arbitrarily small, so that all those cases can be assigned to one particular symmetric extension.    
\end{exam}




  

\section{Discussion and conclusion}
We introduced a functorial definition of projectivity and the generalised distribution semantics, capturing the core idea of the distribution semantics independently of the deterministic framework that is used on top of it.

The main results of the paper showed that all projective families of distributions that can be represented in the generalised distribution semantics satisfy the Strong Independence Property, which restricts models to a stochastic block model construction as well as an additional property that ensures that symmetry is possible.
In practice, one is less interested in the total distribution of a probabilistic logic program and more in its reducts; the free component of the probabilistic logic program is usually seen as ``error terms'' or ``choice predicates'' that are marginalised out, and the deterministic part will often contain auxiliary predicates too.

Consider for instance a ProbLog program for a stochastic block model  modelling smokers being more likely to be friends with other smokers. This could involve the following typical probabilistic clause:

\begin{verbatim}
  0.8 :: friends(X,Y) :- smokes(X), smokes(Y).
\end{verbatim}

To express this in the distribution semantics, the probability annotation is translated to the clauses
\begin{align}
  \mathrm{friends}(X,Y) &\leftarrow \mathrm{smokes}(X), \mathrm{smokes}(Y), \mathrm{u}(X,Y). \\
  0.8 &:: \mathrm{u}(X,Y).
\end{align}

where u is an additional binary predicate added to the language.
The distributions we are ultimately interested in are over worlds in the original language, without the additional predicate u.

Thus the main results give a complete characterisation of the reducts of projective generalised probabilistic logic programs as exactly the reducts of projective families of distributions with the SIP that are not essentially asymmetric.
Indeed, by Theorem \ref{ProjPLPSIP} and Proposition \ref{ProjPLPnotEssAsym}, every reduct of a projective generalised PLP is a reduct of a projective family of distributions with the SIP that is not essentially asymmetric.
By Theorem \ref{SIPnotEAisRed}, every such family of distributions is a reduct of a projective (generalised) PLP, and since the reduct of a reduct is itself a reduct, the claim follows.

Unfortunately SIP is not conserved under reduct, since reducts of traces are usually not traces in the smaller signature. However, the {\em (Constant) Independence Property}, a much studied weaker condition than SIP, is clearly preserved (see \cite{ParisV15,RonelV16} for more background on this notion and the context of the following paragraphs):

\begin{defn}
  A family of distributions satisfies \emph{IP} if the following holds:

  Let $\varphi$ and $\psi$ be quantifier-free $L$ formulas whose variables have been ground to elements of a domain, and who do not mention any joint element. Then the probability of $\varphi \wedge \psi$ is the product of the probabilities of $\varphi$ and $\psi$.
\end{defn}

Therefore every reduct of a projective generalised PLP satisfies IP, which precludes modelling dependencies between the validity of relation symbols for different tuples in a projective way.
In unary signatures, the situation is clearer: There, IP and SIP coincide, and the families that can be modelled are precisely those which are expressible by independently choosing 1-traces for every element according to a prescribed probability for every 1-trace. The restrictiveness of this fragment can be gleaned from de~Finetti's representation theorem, which uniquely represents any projective family of distributions over a unary signature as an {\em infinite mixture\/} of such basic distributions \citep[Chapter 9]{ParisV15}. 

This leads into the next dimension of generalisation: In general, probabilistic logic programming also allows propositions, that is, 0-ary predicates.
These straightforwardly extend the distributions that can be represented as reducts by allowing finite mixtures of the distributions that can be represented without them \citep{Weitkaemper21}.
It is clear that neither of the independence properties generalise to mixtures of distributions.
In the unary case, this allows for representing finite mixtures of the basic distributions. However, the (non-trivial) distributions of ``Carnap's continuum'', the fundamental distributions of unary pure inductive logic, are all infinite mixtures and therefore cannot be represented by generalised PLP \citep[Chapter 16]{ParisV15}.

In the polyadic, the situation is more complicated. Unlike in the unary case, not every projective family of distributions is an infinite mixture of distributions with SIP.\@ Indeed, all infinite mixtures of SIP distributions share a strengthened notion of exchangeability known as {\em signature exchangeability\/} \citep{RonelV16}. In the binary case, it is known that a projective distribution has signature exchangeability if and only if it is an infinite mixture of SIP distributions; in higher arities this is still open.
However, signature exchangeability is not conserved under reducts either, and we do not know of a weaker property implied by it that {\em is\/} conserved.

It is also worthwhile to compare our results with Malhotra and Serafini's recent analysis of projective 2-variable Markov logic networks \cite{MalhotraS22}.
They claim to show that projective Markov logic networks that only use two variables in any formula are precisely the relational block models, or, in other words, precisely those with SIP.\@
However, in a Markov logic network with real-valued weights, every possible world has positive probability, which must therefore be included as an implicit additional assumption to SIP.\@
This additional assumption is in fact stronger than not be essentially asymmetric, and therefore we can conclude that every projective Markov logic network for which every formulas uses at most two variables can be expressed as the reduct of a determinate probabilistic logic program.
Note however that the two-variable assumption is a significant restriction, since it limits attention to at most binary relations and does not allow for the expression of concepts like transitivity.
Indeed, inference in the fragment of two-variable Markov logic networks is tractable and admitting of a closed form solution, while the general problem of inference in Markov logic networks is intractable \cite{MalhotraS22wmc}.
Thus, bringing our understanding of general projective Markov logic networks to the same level as our understanding of (generalised) probabilistic logic programs is a promising avenue for future research. 

As the generalised distribution semantics is formulated abstractly, it captures any possible method to allocate a unique model to every choice of probabilistic facts.  This includes traditional probabilistic logic programming based on Datalog programs with stratified negation as well as any extension of stratified negation that pins down a unique model, for instance logic programs with unique stable models.
Note also that the setting we adopt here requires a relational signature, and therefore does not support function symbols or compound terms in the queries or in the free part of the program.
However, this does not constrain any auxiliary predicates or procedures used within the definition of the logical part of the program.

Other extensions of logic programming such as answer set programming or disjunctive logic programming go beyond this and support programs with several stable models.
There are different approaches for adding probabilistic facts to such extensions.
A recently popular one is the credal semantics proposed by Lukasiewicz \cite{Lukasiewicz07}, which essentially allocates a set of probability distributions corresponding to the stable models of the logic program;
Cozman ans Mau{\'{a}} \cite{CozmanM20} provide an in-depth treatment of how this approach can be carried out for answer set programs.
However, since the credal semantics does not specify a single random world for a domain, it does not result in a family of distributions in the sense discussed here. 

One way to retain a unique probability distribution over models of a domain is to assume the principle of indifference.
This means simply dividing the weight equally between each of the multiple stable models.
This concept has been realised in the languages P-Log \citep{BaralGR09} and Probabilistic Disjunctive Logic Programming \citep{Ngo96} for answer set programming and probabilistic disjunctive logic programming respectively.  

Such a set-up is outside the generalised distribution semantics, since we no longer allocate a unique extension to every intensional world.
Indeed, we can see that it is not generally possible to extend the intensional signature and obtain any P-log program or disjunctive logic program as a reduct, since both can define essentially asymmetric models.
Take for instance the answer set program
\[
R(x,y) \leftarrow \neg R(y,x), x \neq y.
\]
The stable models of this answer set program are precisely those where each two nodes $a$ and $b$ are connected by exactly one directed edge, either from $a$ to $b$ or from $b$ to $a$. By the principle of indifference, every such model is allocated equal probability; however, none of them are symmetric.

The same argument applies to probabilistic disjunctive logic programming, using the disjunctive logic program
\[
R(x,y) \vee R(y,x).
\]
in place of the answer set program given above.

This shows that beyond the limitations of the concrete fragment of Prolog implemented in a probabilistic logic programming language,
the foundational underpinnings of the distribution semantics as a deterministic program over probabilistic facts is itself restrictive of the families of distributions that can be represented.
We have seen that in the case of projective families of distributions, particularly simple families that allow for marginal inference independent of the domain size, improving the power of the logic used to define the deterministic component does not increase expressivity beyond quantifier-free first-order formulas (or equivalently determinate logic programs).
Furthermore, by giving an explicit description of the expressible projective families as reducts of relational stochastic block models, we see that the vast majority of projective families remain out of reach of formalisms based directly on the distribution semantics.
  \bibliographystyle{elsarticle-num} 
  \bibliography{Probabilisticlogicbib}
\end{document}